\documentclass[10pt]{article}
\usepackage[utf8]{inputenc}
\usepackage[T1]{fontenc}
\usepackage[margin=1in, left=1.25in, right=1.25in, top=1in, bottom=1in]{geometry}
\usepackage{algorithm}
\usepackage{algorithmic}
\usepackage[linktoc=all, bookmarksdepth=3, pdfstartview=FitH]{hyperref}
\hypersetup{colorlinks=true, linkcolor=blue, filecolor=magenta, urlcolor=cyan}
\usepackage{amsmath}
\usepackage{amsfonts}
\usepackage{amssymb}
\usepackage{mathrsfs}
\usepackage{bbold}
\usepackage{graphicx}
\usepackage[export]{adjustbox}
\usepackage[dvipsnames]{xcolor}
\usepackage{float}
\usepackage{placeins}
\usepackage{caption}
\usepackage{booktabs}
\usepackage{multirow}
\usepackage[normalem]{ulem}
\usepackage{amsthm}
\usepackage{enumitem}
\usepackage{microtype}
\newtheorem{theorem}{Theorem}[section]
\newtheorem{lemma}[theorem]{Lemma}
\newtheorem{definition}[theorem]{Definition}
\newtheorem{corollary}[theorem]{Corollary}
\newtheorem{proposition}[theorem]{Proposition}
\newtheorem{remark}[theorem]{Remark}
\newtheorem{assumption}[theorem]{Assumption}

\usepackage{cleveref}
\crefname{assumption}{Assumption}{Assumptions}
\crefname{proposition}{Proposition}{Propositions}
\crefname{theorem}{Theorem}{Theorems}
\crefname{lemma}{Lemma}{Lemmas}
\crefname{remark}{Remark}{Remarks}
\crefname{algorithm}{Algorithm}{Algorithms}
\crefname{equation}{Eq.}{Eqs.}

\graphicspath{ {./} }

\title{Arithmetic Actor Heads and Training Stabilization for Out-of-Distribution Reinforcement Learning}

\author{
  Yifan Zhang\thanks{Central South University, \href{mailto:erzhu419@gmail.com}{erzhu419@gmail.com}, \href{mailto:204201048@csu.edu.cn}{204201048@csu.edu.cn}} \and
  Liang Zheng\thanks{Central South University, \href{mailto:zhengliang@csu.edu.cn}{zhengliang@csu.edu.cn}} \and
  Qifan Zhang\thanks{Monash Malaysia Faculty of IT, \href{mailto:zhangqifan0419@gmail.com}{zhangqifan0419@gmail.com}}
}
\date{}

\let\svthefootnote\thefootnote
\newcommand\blfootnotetext[1]{%
  \let\thefootnote\relax\footnote{#1}%
  \addtocounter{footnote}{-1}%
  \let\thefootnote\svthefootnote%
}
\let\svfootnotetext\footnotetext
\renewcommand\footnotetext[2][?]{%
  \if\relax#1\relax%
    \ifnum\value{footnote}=0\blfootnotetext{#2}\else\svfootnotetext{#2}\fi%
  \else%
    \if?#1\ifnum\value{footnote}=0\blfootnotetext{#2}\else\svfootnotetext{#2}\fi%
    \else\svfootnotetext[#1]{#2}\fi%
  \fi
}

\begin{document}
\maketitle

\begin{abstract}
Reinforcement learning (RL) policies can deteriorate under out-of-distribution (OOD) magnitude shifts. Starting from soft actor-critic (SAC) and its Bayesian Amnesic Piecewise-Robust (BAPR) predecessor, we study the causal-symbolic BAPR (CS-BAPR) family. The practical method combines six training-stabilization settings with alternative actor heads: a Neural Addition Unit (NAU) with a Neural Multiplication Unit (NMU)-inspired quadratic correction, a Kolmogorov-Arnold Network (KAN), or a multilayer perceptron (MLP) with rectified linear unit (ReLU) or hyperbolic-tangent activations.

For a differentiable deterministic policy map $\pi$ and an action-valued reference $\pi_{\mathrm{ref}}$, we prove the conditional bound
$\|\pi(x_{\mathrm{ood}})-\pi_{\mathrm{ref}}(x_{\mathrm{ood}})\|
\leq \delta+\varepsilon\|d\|+\tfrac{L_{\mathrm{err}}}{2}\|d\|^2$.
Here $x_{\mathrm{ood}}$ is a test state, $d=x_{\mathrm{ood}}-x_0$ is displacement from a training reference state $x_0$, $\delta$ and $\varepsilon$ bound the value and derivative errors at $x_0$, and $L_{\mathrm{err}}$ bounds derivative-error growth along the connecting segment. An isolated NAU is linear and has zero derivative curvature, whereas scalar ReLU has a derivative jump. These layer facts do not certify the historical actor, which contains nonlinear features, a quadratic branch, and action squashing. A revised smooth-feature variant admits an explicit whole-policy derivative bound. A separate frozen-operator result preserves BAPR's discounted Bellman contraction.

On a 12-line transit benchmark calibrated with Simulation of Urban MObility (SUMO), training uses nominal in-distribution (ID) demand. All four stabilized variants improve reported median returns over the unstabilized BAPR implementation under static, abrupt, and oscillating demand shifts up to $50\times$ nominal flow. Stabilization accounts for most of the gain; the NAU variant is best at the static $5\times$ point, while other heads perform better in several remaining regimes. Optional sparse identification of nonlinear dynamics (SINDy) and invariant risk minimization (IRM)-inspired components are described separately. Auxiliary linear-quadratic regulator (LQR) and Hopper experiments show task-dependent behavior. The evidence supports a transit-focused training recipe and a conditional policy-output analysis.
\end{abstract}

\noindent\textbf{Keywords:} OOD extrapolation; arithmetic actor heads; training stabilization; transit holding; robust RL.

\section{Introduction}

Policies trained under nominal conditions may perform poorly when passenger demand or other exogenous inputs increase beyond the training range. We call this setting \emph{magnitude extrapolation}. SAC provides the underlying continuous-control training framework~\cite{Haarnoja2018SAC}; BAPR supplies change-sensitive adaptation~\cite{Zhang2025BAPR}, and robust ensemble SAC (RE-SAC) supplies the critic ensemble and risk regularization~\cite{Zhang2025RESAC}.

Magnitude extrapolation connects two questions: whether training produces a stable controller, and whether the resulting policy responds appropriately to larger inputs. Arithmetic layers offer a useful inductive bias for numerical extrapolation~\cite{Trask2018NALU,Madsen2020NAU}, but a linear output layer alone does not make a nonlinear policy linear in the original state. We therefore distinguish the layer-level calculation, the conditional whole-policy error bound, and the observed control returns.

CS-BAPR combines four base stabilization settings with two recipes ported from BAPR: minimum-ensemble critic targets, reward-scale normalization, an entropy-temperature floor, removal of the behavior cloning (BC) anchor for stabilized variants, change-detection warmup and penalty scaling, and rollout-based one-sided surprise. Bayesian online changepoint detection (BOCD) provides the belief tracker~\cite{Adams2007BOCD}. The architectural variants share this training backbone.

The contributions are:
\begin{itemize}
    \item A transit-focused comparison of four actor-head choices under the same stabilization recipe, against an unstabilized and a legacy BAPR implementation.
    \item A policy-output extrapolation bound with explicit boundary errors and derivative-growth assumptions (\Cref{thm:ood_quad_nD}), together with layer-level NAU, quadratic-unit, and ReLU calculations (\Cref{sec:nau}).
    \item A frozen Bellman-operator contraction result (\Cref{thm:csbapr_contraction}) and a counterexample showing how sufficiently strong value-dependent penalties can destroy contraction (\Cref{thm:frozen_sym_necessary}).
    \item Conditional symbolic-reference results that state the required action-space alignment and cross-environment accuracy explicitly (\Cref{sec:sindy,sec:irm}).
\end{itemize}
The complete proofs and their dependencies are linked in Appendix~\ref{app:lean_overview}; implementation details appear in Appendices~\ref{app:operator_detail}--\ref{app:hyperparams}.

\section{Related work}

\subsection{Generalization and robust RL}
Simulation-trained locomotion can generalize to previously unseen terrain~\cite{Lee2020ESCP}. Robust Markov decision processes (MDPs) instead optimize performance over specified transition uncertainty sets~\cite{Wiesemann2013RobustMDP,Iyengar2005RobustDP}. Recent robust actor-critic and fitted-value methods extend such analyses to function approximation and offline or hybrid data~\cite{Zhou2023NaturalActorCritic,Panaganti2024RobustF}. Their guarantees depend on their uncertainty and coverage assumptions; they are not generally restricted to interpolation in a convex hull of training environments. Epistemic partially observable MDPs explain why uncertainty about test conditions can require adaptive behavior even when each individual environment is fully observed~\cite{Ghosh2021EpistemicPOMDP}. Our bound concerns the output of a fixed deterministic policy map along a state-space segment.

\subsection{Arithmetic networks and equation learning}
Neural Arithmetic Logic Unit (NALU) networks introduce gated arithmetic computations for numerical extrapolation~\cite{Trask2018NALU}. NAUs and NMUs separate addition and multiplication and use constrained weights with sparsity regularization~\cite{Madsen2020NAU}. The NAU is an addition/subtraction unit; the paper title \emph{Neural Arithmetic Units} denotes the broader family. Equation-learning networks recover analytic relations for extrapolation and control~\cite{Sahoo2018LearningEquations}; this is a distinct architecture rather than an application of NAUs. KANs use learnable spline functions on edges~\cite{Liu2024KAN}. We compare these architectural ideas in one training framework without claiming priority for applying arithmetic structure to control.

\subsection{Symbolic dynamics and invariance}
SINDy identifies sparse governing dynamics from a candidate library~\cite{Brunton2016SINDy}; SINDy-RL combines interpretable dynamics, reward, and policy models with RL~\cite{Zolman2025SINDyRL}. AI Feynman exploits structural properties of equations for symbolic regression~\cite{Udrescu2020AIFeynman}. These works motivate equation identification but do not imply that a dynamics Jacobian equals an optimal-policy Jacobian.

IRM seeks representations whose optimal predictor is shared across training environments~\cite{Arjovsky2019IRM}. Causal representation learning and causal RL involve additional assumptions about interventions, latent variables, and causal structure~\cite{Scholkopf2021CausalRL,CausalRLSurvey2023}; structural causal models (SCMs) provide a separate formal account~\cite{Peters2017Elements}. Our optional selector is inspired by cross-environment invariance and uses prediction-error variance. Its mathematical results assume the stated derivative accuracy; they do not establish causal discovery from the selector alone.

\subsection{Derivative regularity and ensemble stabilization}
Rowan analyzes ReLU failures in physics-informed objectives, including mis-specified automatic-differentiation gradients of discontinuous fields~\cite{ReLUFailure2023}. That result concerns physics-informed training, rather than a general inability of ReLU networks to approximate smooth functions or extrapolate. Here we use the elementary derivative jump at the scalar ReLU kink to delimit a smooth-derivative certificate. Ensemble pessimism also appears in offline RL~\cite{An2021UncertaintyEnsemble}; this motivates the minimum-ensemble comparison but does not transfer an offline performance guarantee to our online controller. Our Bellman analysis follows the frozen-parameter template of RE-SAC and BAPR~\cite{Zhang2025RESAC,Zhang2025BAPR}.

\section{Scope and applicability}
\label{sec:scope}

The primary empirical setting is transit holding under changes in passenger-demand magnitude. The six stabilization settings and the actor-head choice are evaluated jointly, with symbolic losses disabled. A demand multiplier changes the distribution of visited states; it is not itself the state-space displacement appearing in the theorem.

The optional symbolic extension requires an action-valued reference constructed from an identified model or other control information. A sparse dynamics model alone is insufficient: its state-evolution output has different physical units and generally a different dimension from the policy action. Tracking and counteraction tasks can both admit model-based feedback, but each requires an explicit dynamics-to-control construction. The transit results do not validate such a construction.

Cross-environment filtering additionally requires environments that reveal the invariant relation of interest. The configurable three-environment default is an implementation choice, not a universal identifiability threshold from IRM. The empirical and theoretical limits are collected in \Cref{sec:limitations}.

\section{Preliminaries}
\label{sec:preliminaries}

\paragraph{Notation.}
Let $E=\mathbb{R}^{d_x}$ and $F=\mathbb{R}^{d_a}$ be the ambient state and action spaces, where $d_x$ and $d_a$ are their dimensions. The admissible spaces are $\mathcal{S}\subseteq E$ and $\mathcal{A}\subseteq F$. We use $x$ for a state, $a$ for an action, $t$ for environment time, and $j$ for a training update. The stochastic actor has density $p_\theta(a\mid x)$, where $\theta$ denotes actor parameters. Its pre-squash Gaussian mean is $\mu_\theta(x)$; $\pi(x)$ denotes the deterministic action map after squashing and scaling. Thus $p_\theta$, $\mu_\theta$, and $\pi$ name different objects.

The training domain is $D_{\mathrm{train}}$; $\Omega\subseteq E$ is the convex analysis domain containing the segment of interest. $\pi_{\mathrm{ref}}:\Omega\to F$ is an action-valued reference. Write $d=x_{\mathrm{ood}}-x_0$ for displacement from a fixed $x_0\in D_{\mathrm{train}}$ to an OOD state $x_{\mathrm{ood}}$. $\mathrm{D}$ denotes the Fr\'echet derivative; vector norms are Euclidean and derivative norms are induced operator norms unless a Frobenius norm $\|\cdot\|_{\mathrm F}$ is specified. $\|Q\|_\infty$ denotes the supremum norm of a state-action value function.

The nonnegative constants $\delta$ and $\varepsilon$ bound reference-value and derivative error at $x_0$. $L_\pi$ and $M$ bound the derivative Lipschitz constants of $\pi$ and $\pi_{\mathrm{ref}}$, and $L_{\mathrm{err}}=L_\pi+M$ bounds their derivative-error growth. A subscript identifies the object whose regularity is bounded; the constant for a generic head $H$ is $L_H$. The operators $\mathbb E$ and $\operatorname{Var}$ denote expectation and variance.

\subsection{Maximum-entropy objective}

Following SAC~\cite{Haarnoja2018SAC}, the stochastic actor maximizes
\begin{equation}
J_{\mathrm{ent}}(p_\theta)=
\mathbb{E}_{p_\theta,P}\left[
\sum_{t=0}^{\infty}\gamma^t
\bigl(R(x_t,a_t)+\alpha\mathcal{H}(p_\theta(\cdot\mid x_t))\bigr)
\right].
\label{eq:max_entropy}
\end{equation}
Here $P$ is the transition kernel, $R$ is the one-step reward, $0\leq\gamma<1$ is the discount factor, $\alpha>0$ is the entropy temperature, and $\mathcal{H}(p)=-\int p(a)\log p(a)\,\mathrm d a$ is differential entropy for continuous actions. The experiments report unaugmented episode return, not the objective \eqref{eq:max_entropy}.

\subsection{BAPR backbone}

Let $\mathcal{M}$ be the finite mode set, $\ell\in\mathcal{M}$ a mode, and $\rho(\ell)$ its BOCD-derived belief probability. BAPR uses a frozen mixture of mode-conditional operators~\cite{Zhang2025BAPR},
\begin{equation}
\mathcal{T}^{\mathrm{BAPR}}_\rho Q(x,a)
=\sum_{\ell\in\mathcal{M}}\rho(\ell)\mathcal{T}_\ell Q(x,a).
\label{eq:bapr_operator}
\end{equation}
RE-SAC supplies aleatoric regularization and epistemic ensemble penalties~\cite{Zhang2025RESAC}. The contraction analysis of the mixture \eqref{eq:bapr_operator} fixes the actor density, entropy temperature, belief, and penalty values while comparing two value functions. It is an operator result for a fixed backup.

\section{Methodology}

\subsection{The extrapolation comparison}
\label{sec:ood_problem}

We compare $\pi$ with $\pi_{\mathrm{ref}}$ on the segment joining $x_0$ and $x_{\mathrm{ood}}$. In the default analysis, the reference is the affine tangent
\begin{equation}
\pi_{\mathrm{ref}}(x)=\pi(x_0)+\mathrm{D}\pi(x_0)(x-x_0).
\label{eq:tangent_reference}
\end{equation}
The tangent reference \eqref{eq:tangent_reference} gives $\delta=\varepsilon=M=0$ at $x_0$. The resulting bound measures deviation from the policy's own tangent; closeness to an optimal policy requires a separate reference.

\begin{assumption}[Reference and boundary accuracy]
\label{assump:reference}
The policy and reference are continuously differentiable on a neighborhood of the segment, with
$\|\pi(x_0)-\pi_{\mathrm{ref}}(x_0)\|\leq\delta$ and
$\|\mathrm{D}\pi(x_0)-\mathrm{D}\pi_{\mathrm{ref}}(x_0)\|\leq\varepsilon$.
\end{assumption}

\begin{assumption}[Reference derivative regularity]
\label{assump:smooth}
For $x,y\in\Omega$,
$\|\mathrm{D}\pi_{\mathrm{ref}}(x)-\mathrm{D}\pi_{\mathrm{ref}}(y)\|
\leq M\|x-y\|$, with finite $M\geq0$.
An affine reference has $M=0$.
\end{assumption}

\begin{assumption}[Whole-policy derivative regularity]
\label{assump:feature_lip}
For $x,y\in\Omega$,
$\|\mathrm{D}\pi(x)-\mathrm{D}\pi(y)\|\leq L_\pi\|x-y\|$,
with finite $L_\pi\geq0$.
This is a condition on the complete deterministic action map, including its feature extractor and output transformations.
\end{assumption}

\begin{definition}[Jacobian consistency (JC)]
\label{def:jac_consistency}
For an action-valued symbolic reference $\pi_{\mathrm{sym}}:\Omega\to F$ and nonnegative alignment-error bound $\varepsilon_2$, $\pi$ is $\varepsilon_2$-Jacobian consistent on $\Omega$ when
$\|\mathrm{D}\pi(x)-\mathrm{D}\pi_{\mathrm{sym}}(x)\|\leq\varepsilon_2$
for every $x\in\Omega$.
\end{definition}

\begin{definition}[Derivative regularity]
\label{def:arith_bias}
A differentiable head $H$ has derivative Lipschitz constant $L_H$ when the norm of its derivative difference is at most $L_H$ times input distance. We use this property to quantify curvature; it is shared by many smooth non-arithmetic maps.
\end{definition}

\subsection{Arithmetic and alternative actor heads}
\label{sec:nau}

\paragraph{NAU.}
An NAU~\cite{Madsen2020NAU} computes
$H_{\mathrm{NAU}}(z)=W_{\mathrm{NAU}}z$,
where $z\in\mathbb{R}^{d_z}$ is a feature vector, $d_z$ is feature dimension, and $W_{\mathrm{NAU}}$ is a fixed weight matrix at evaluation. Training clamps weights to $[-1,1]$ and encourages the entries toward $\{-1,0,1\}$.

\begin{theorem}[NAU derivative constancy]
\label{thm:nau_const}
For every fixed real matrix $W_{\mathrm{NAU}}$,
$\mathrm{D}H_{\mathrm{NAU}}(z)=W_{\mathrm{NAU}}$ for all $z$, so
$L_{H_{\mathrm{NAU}}}=0$.
\end{theorem}
\begin{proof}
Differentiation of a linear map gives its fixed matrix. The derivative difference between any two feature vectors is zero. Discrete weights are unnecessary for this derivative identity.
\end{proof}

\paragraph{Simplified quadratic correction.}
The published NMU multiplies selected input coordinates~\cite{Madsen2020NAU}. Our implementation uses the simpler scalar quadratic map
$H_{\mathrm{quad}}(z)=c_{\mathrm{quad}}z^2$, with fixed real coefficient $c_{\mathrm{quad}}$, as an NMU-inspired correction.

\begin{theorem}[Scalar quadratic derivative regularity]
\label{thm:nmu_bias}
$H_{\mathrm{quad}}'(z)=2c_{\mathrm{quad}}z$ has derivative Lipschitz constant $2|c_{\mathrm{quad}}|$.
A subsequent fixed scalar NAU weight $w_{\mathrm{NAU}}$ gives constant $2|w_{\mathrm{NAU}}c_{\mathrm{quad}}|$.
\end{theorem}
\begin{proof}
$|H_{\mathrm{quad}}'(z_1)-H_{\mathrm{quad}}'(z_2)|
=2|c_{\mathrm{quad}}||z_1-z_2|$. Multiplication by $w_{\mathrm{NAU}}$ scales both sides by $|w_{\mathrm{NAU}}|$.
\end{proof}

\begin{theorem}[Scalar ReLU derivative jump]
\label{thm:relu_impossible}
The scalar map $H_{\mathrm{ReLU}}(z)=\max(0,z)$ has no finite Lipschitz derivative on any interval containing negative and positive inputs.
\end{theorem}
\begin{proof}
For $z<0$ the derivative is zero; for $z>0$ it is one. If a finite $L_{H_{\mathrm{ReLU}}}>0$ bounded the derivative difference, taking
$z_1=-q$ and $z_2=q$ with $0<q<(2L_{H_{\mathrm{ReLU}}})^{-1}$
would require $1\leq2L_{H_{\mathrm{ReLU}}}q<1$.
Here $q$ is a positive test offset. $L_{H_{\mathrm{ReLU}}}=0$ is also impossible, and the derivative is undefined at zero.
\end{proof}
This statement does not rule out affine ReLU networks, segments staying in one activation region, or piecewise analyses. Smooth MLPs and sufficiently smooth spline networks can satisfy the whole-map condition on a specified domain; the constant also depends on their weights and output transformations.

\begin{remark}[Composition requires derivative control of the inner map]
\label{rem:theory_practice}
Write the deterministic normalized action as
$\pi=\tanh\circ H_{\mathrm{mix}}\circ G$, where $G:E\to\mathbb R^{d_z}$ is the feature extractor and $H_{\mathrm{mix}}:\mathbb R^{d_z}\to F$ is the mixed arithmetic head.
If $G$ and a head $H$ are continuously differentiable, their derivative norms are bounded by $K_G,K_H$, and their derivative Lipschitz constants are $L_G,L_H$, then
\begin{equation}
L_{H\circ G}\leq L_HK_G^2+K_HL_G.
\label{eq:composition_rule}
\end{equation}
Thus a function-Lipschitz inner map is insufficient. LeakyReLU features have derivative jumps across activation boundaries; a tanh output also introduces curvature and saturation. Appendix~\ref{app:leff} proves \eqref{eq:composition_rule}, gives the action-squashing term in \eqref{eq:complete_composition}, and derives the revised smooth-feature certificate. The isolated NAU constant cannot be substituted for $L_\pi$ of the historical actor.
\end{remark}
Appendix~\ref{app:nau_details} gives the architecture and weight regularizer \eqref{eq:nau_regularizer}. In the historical comparison, \texttt{csbapr-tanh} uses smooth MLP activations; \texttt{csbapr-no-nau} retains ReLU hidden activations.

\subsection{Training stabilization}
\label{sec:training_fixes}

We use six settings in the BAPR training backbone. Exponential moving average (EMA) normalization stabilizes reward scale, while temporal-difference (TD) targets and entropy tuning retain the SAC structure. The stabilized variants also use actor weight decay, with head-specific regularizers where applicable. These settings are evaluated as a package; the tables do not identify an individual setting's causal contribution.

\paragraph{Four base settings.}
\begin{enumerate}[label=(A\arabic*),leftmargin=2em]
    \item \textbf{Minimum-ensemble target.} Replace per-critic next-state targets by
    $Q_{\min}(x',a')=\min_{1\leq k\leq N_{\mathrm{crit}}}Q_{\bar\phi_k}(x',a')$,
    where $N_{\mathrm{crit}}$ is ensemble size, $\phi_k$ are critic parameters, and $\bar\phi_k$ are target-network parameters. The full target is given in \eqref{eq:critic_target}.
    \item \textbf{Reward normalization.} At training update $j$, set
    $\sigma_{r,j+1}^2=0.99\sigma_{r,j}^2+0.01\operatorname{Var}(r_{\mathrm{batch}})$ and divide rewards by $\max(\sigma_{r,j+1},10^{-6})$. Here $\sigma_{r,j}$ is the running reward scale and $r_{\mathrm{batch}}$ is a replay-batch reward vector.
    \item \textbf{Entropy floor.} Enforce $\alpha\geq\alpha_{\min}=0.01$ during temperature tuning.
    \item \textbf{BC removal.} In the current transit script, set the BC coefficient $\beta_{\mathrm{bc}}=0$ for all four stabilized variants. The pristine baseline retains $\beta_{\mathrm{bc}}=0.001$. Thus BC removal is part of the stabilization comparison, not a difference between the four stabilized heads.
\end{enumerate}

\paragraph{Two BAPR recipes.}
\begin{enumerate}[label=(B\arabic*),leftmargin=2em]
    \item \textbf{Warmup and penalty scaling.} Set the belief-weighted penalty strength $\lambda_w$ to zero during the first 100 updates, then use
    $\beta_{\mathrm{eff}}=-2-c_{\mathrm{pen}}\lambda_w$
    with $c_{\mathrm{pen}}=2$ rather than 5. Here $\beta_{\mathrm{eff}}$ is the actor ensemble-pessimism coefficient and $c_{\mathrm{pen}}$ its scale. At $\lambda_w=0.3$, the coefficients are $-2.6$ and $-3.5$, respectively.
    \item \textbf{Recent-rollout, one-sided surprise.} Compute BOCD surprise on the latest $N_{\mathrm{tail}}=256$ rollout transitions when supplied by the training script. For the ensemble-spread signal $\sigma_{Q,t}$, use
    $\max(\sigma_{Q,t}/\sigma_{Q,t-1}-1,0)$
    instead of a two-sided absolute change. The implementation stabilizes the ratio numerically. The scalar surprise is denoted $\chi_t$, reserving symbolic coefficients for equation identification.
\end{enumerate}
Table~\ref{tab:hyperparams} in Appendix~\ref{app:hyperparams} lists implementation defaults and overrides, and \Cref{alg:csbapr} shows their place in the training loop.

\subsection{Optional symbolic-reference extension}
\label{sec:sindy}

A dynamics model and a policy reference have different roles. Let
$f_{\mathrm{dyn}}(x,a)$ be the conditional mean one-step state increment and
$f_{\mathrm{sym}}(x,a)$ its identified counterpart, both with values in $E$.
This convention matches the implemented discrete fitting target.
For a state-only fit, the action argument is absent and the fitted relation describes data collected under the exploration policy.

SINDy uses a scalar basis library
$\boldsymbol{\varphi}=(\varphi_1,\ldots,\varphi_{N_{\mathrm{lib}}})$,
where $N_{\mathrm{lib}}$ is the number of basis terms. With coefficient matrix $\Xi$,
$f_{\mathrm{sym}}(x,a)=\Xi\boldsymbol{\varphi}(x,a)$.
The regression data comprise joint state-action input matrix $\mathsf X$, output matrix $\mathsf Y$, and library design matrix $\boldsymbol{\Phi}(\mathsf X)$. The implementation uses sequentially thresholded least squares (STLSQ)~\cite{Brunton2016SINDy}. The revised fitting and frozen wrapper both consume actions; Appendix~\ref{app:sindy_details} specifies their alignment and state-increment target.

\begin{assumption}[Action-valued symbolic reference]
\label{assump:sindy}
For the optional theorem, a control construction $\mathcal{C}$ produces an action-valued reference
$\pi_{\mathrm{sym}}=\mathcal{C}(f_{\mathrm{sym}}):\Omega\to F$.
Its derivative error relative to $\pi_{\mathrm{ref}}$ is bounded by $\varepsilon_1$ on the stated domain. This condition must be established after the control construction; it does not follow solely from a small dynamics prediction error.
\end{assumption}

For a homogeneous linear discrete fit, the revised LQR construction uses
$f_{\mathrm{sym}}(x,a)=(\widehat A-I)x+\widehat B_{\mathrm{in}}a$,
where $\widehat A$ and $\widehat B_{\mathrm{in}}$ are the identified state and input matrices and $I$ is the identity matrix.
Given state and normalized-action cost matrices $Q_{\mathrm{cost}}$ and $R_{\mathrm{cost}}$, a discrete algebraic Riccati solver determines feedback gain $\widehat K_{\mathrm{fb}}$.
The action-valued reference is $\pi_{\mathrm{sym}}(x)=\tanh(-\widehat K_{\mathrm{fb}}x)$, with componentwise tanh.
It is a smooth bounded reference, rather than the constrained LQR optimum.
Its derivative Lipschitz constant is at most
$\frac{4}{3\sqrt3}\|\widehat K_{\mathrm{fb}}\|_{\mathrm{op}}^2$.
The input matrices come from collected transitions, not from the simulator's true matrices; only the task costs enter the control construction. Symbolic penalties require this action-valued construction and are not applied directly to a state-valued dynamics model.

\begin{theorem}[Symbolic-reference error composition]
\label{thm:sindy_jac}
On the same domain and against the same identified reference, if
$\|\mathrm{D}\pi_{\mathrm{sym}}-\mathrm{D}\pi_{\mathrm{ref}}\|\leq\varepsilon_1$
and
$\|\mathrm{D}\pi-\mathrm{D}\pi_{\mathrm{sym}}\|\leq\varepsilon_2$,
then $\|\mathrm{D}\pi-\mathrm{D}\pi_{\mathrm{ref}}\|\leq\varepsilon_1+\varepsilon_2$.
\end{theorem}
\begin{proof}
Insert $\mathrm{D}\pi_{\mathrm{sym}}$ and apply the triangle inequality. The domain and both codomains must agree.
\end{proof}

To encourage the alignment in \Cref{def:jac_consistency}, the ideal action-valued JC loss is
\begin{equation}
\mathcal{L}_{\mathrm{JC}}=
\mathbb{E}_{x\sim\mathcal{D}}\!
\left[\|\mathrm{D}\pi(x)-\mathrm{D}\pi_{\mathrm{sym}}(x)\|_{\mathrm F}^{2}\right],
\label{eq:jac_loss}
\end{equation}
where $\mathcal{D}$ denotes the replay sampling distribution. The revised implementation evaluates the full deterministic-action Jacobian, with one row per action coordinate, and uses the squared Frobenius loss in \eqref{eq:jac_loss}. Symbolic Jacobians are detached; the critic uses a detached batch scalar and the actor uses the unclipped loss with gradient-norm clipping. The historical symbolic implementation instead compared summed state-increment gradients with the pre-squash actor mean. Its scalar maps were
$\psi_{\mathrm{pol}}(x)=\sum_{i_a=1}^{d_a}\mu_{\theta,i_a}(x)$ and
$\psi_{\mathrm{sym}}(x)=\sum_{i_x=1}^{d_x}f_{\mathrm{sym},i_x}(x)$,
where $i_a$ and $i_x$ index action and state coordinates, respectively.
Its historical loss was
\begin{equation}
\mathcal{L}_{\mathrm{proj}}=
\mathbb{E}_{x\sim\mathcal{D}}\!
\left[\frac{\|\nabla_x\psi_{\mathrm{pol}}(x)
-\nabla_x\psi_{\mathrm{sym}}(x)\|^2}{d_x}\right].
\label{eq:projected_loss}
\end{equation}
The historical diagnostic \eqref{eq:projected_loss} does not establish full action-valued JC. The revised loss removes coordinate cancellation and matches the action units of the reference. Primary transit runs disable both symbolic weights.

\begin{theorem}[Conditional symbolic-to-OOD pipeline]
\label{thm:sindy_pipeline}
Under the two derivative-error bounds in \Cref{thm:sindy_jac} everywhere on the segment and base error at most $\delta$,
\begin{equation}
\|\pi(x_{\mathrm{ood}})-\pi_{\mathrm{ref}}(x_{\mathrm{ood}})\|
\leq\delta+(\varepsilon_1+\varepsilon_2)\|d\|.
\label{eq:symbolic_linear}
\end{equation}
If only the boundary errors are known, the linear guarantee \eqref{eq:symbolic_linear} is replaced by \Cref{thm:ood_quad_nD}, adding the term $L_{\mathrm{err}}\|d\|^2/2$ under its regularity assumptions.
\end{theorem}
The linear result uses uniform segment-wide derivative accuracy; it does not require $L_\pi=0$. Its proof is in Appendix~\ref{app:proof_symbolic}. A replay-average loss alone supplies neither a uniform boundary certificate nor an OOD guarantee.

\subsection{Optional cross-environment filtering}
\label{sec:irm}

Let $e\in\mathcal{E}_{\mathrm{train}}$ index a training environment, where $N_{\mathrm{env}}=|\mathcal{E}_{\mathrm{train}}|$. Write $\pi_{\mathrm{ref}}^e$ for its action-valued reference. IRM~\cite{Arjovsky2019IRM} motivates looking for relations shared across environments. The current code selects symbolic models using variance of \emph{prediction errors}; it does not implement the gradient-penalty formulation of IRM (IRMv1) in this selector. Appendix~\ref{app:irm_protocol} gives the exact criterion in \eqref{eq:invariance_selector}.

\begin{definition}[Cross-environment derivative accuracy]
\label{def:irm}
A shared action-valued reference $\pi_{\mathrm{sym}}$ is
$\varepsilon_{\mathrm{irm}}$-accurate when
$\|\mathrm{D}\pi_{\mathrm{ref}}^e(x)-\mathrm{D}\pi_{\mathrm{sym}}(x)\|
\leq\varepsilon_{\mathrm{irm}}$ on every stated environment domain.
\end{definition}

\begin{proposition}[Conditional tightening]
\label{thm:irm_tighten}
If a selected reference satisfies
$\varepsilon_{\mathrm{irm}}\leq\varepsilon_{\mathrm{raw}}$,
where $\varepsilon_{\mathrm{raw}}$ is the corresponding single-environment derivative-error bound, adding the same policy-alignment error $\varepsilon_2$ preserves this ordering. Selection by error variance alone does not imply the premise.
\end{proposition}
\begin{proof}
Add $\varepsilon_2$ to both sides of the assumed inequality.
\end{proof}

\begin{proposition}[Invariant-component decomposition]
\label{thm:irm_causal}
Suppose
$\pi_{\mathrm{ref}}^e=\pi_{\mathrm{inv}}+\pi_{\mathrm{var}}^e$,
where $\pi_{\mathrm{inv}}$ is the environment-invariant component and
$\pi_{\mathrm{var}}^e$ is an environment-dependent residual.
Assume an invariant model $\pi_{\mathrm{sym}}^{\mathrm{inv}}$ approximates $\pi_{\mathrm{inv}}$ with derivative error at most $\varepsilon_{\mathrm{inv}}$ on the test domain.
Also assume a single-environment model $\pi_{\mathrm{sym}}^{\mathrm{single}}$ approximates
$\pi_{\mathrm{inv}}+\pi_{\mathrm{var}}^{e_0}$ with the same error bound on that test domain, where $e_0$ is its training environment.
At test environment $e_\star$, the respective upper bounds are
\[
\varepsilon_{\mathrm{inv}}+\varepsilon_{\mathrm{var}}^{e_\star}
\quad\text{and}\quad
\varepsilon_{\mathrm{inv}}+
\varepsilon_{\mathrm{var}}^{e_0}+
\varepsilon_{\mathrm{var}}^{e_\star},
\]
where $\varepsilon_{\mathrm{var}}^e$ bounds
$\|\mathrm{D}\pi_{\mathrm{var}}^e(x)\|$ on the test domain.
\end{proposition}
The difference between these upper bounds is not a lower bound on the realized benefit. The proof and that distinction appear in Appendix~\ref{app:proof_irm}; identifying the invariant component remains an assumption.

\begin{corollary}[Conditional filtering-to-OOD pipeline]
\label{thm:irm_pipeline}
If \Cref{def:irm} and policy alignment with error $\varepsilon_2$ hold on the whole test segment in environment $e_\star$, then
$\|\pi(x_{\mathrm{ood}})-\pi_{\mathrm{ref}}^{e_\star}(x_{\mathrm{ood}})\|
\leq\delta+(\varepsilon_{\mathrm{irm}}+\varepsilon_2)\|d\|$,
with base error $\delta$ at $x_0$ in the same environment.
\end{corollary}
Appendix~\ref{app:proof_irm} supplies the proof. The environment-diversity requirement depends on the model class and interventions; IRM does not imply a universal three-environment sufficiency rule.

\subsection{Frozen Bellman operator}
\label{sec:operator}

Fix a stochastic actor density $p_\theta$ and define
\begin{equation}
V^Q(x)=\int_{\mathcal{A}}p_\theta(a\mid x)
\bigl(Q(x,a)-\alpha\log p_\theta(a\mid x)\bigr)\,\mathrm d a.
\label{eq:soft_value}
\end{equation}
Assume its entropy contribution and all frozen reward and penalty terms are bounded. For mode $\ell$, transition kernel $P_\ell$ and reward $R_\ell$, define
\begin{equation}
\begin{split}
\mathcal{T}^{\mathrm{CS}}_\ell Q(x,a)=R_\ell(x,a)+\gamma\bigg[
&\int_{\mathcal{S}}V^Q(x')P_\ell(dx'\mid x,a)\\
&-\lambda_{\mathrm{epi}}\Gamma_{\mathrm{epi}}^\ell(x,a)
-\lambda_{\mathrm{sym}}\Gamma_{\mathrm{sym}}^\ell(x,a)
-\kappa_{\mathrm{ale}}^\ell(x,a)\bigg].
\end{split}
\label{eq:csbapr_mode}
\end{equation}
Here $\lambda_{\mathrm{epi}},\lambda_{\mathrm{sym}}\geq0$ are penalty coefficients;
$\Gamma_{\mathrm{epi}}^\ell$ is a fixed ensemble uncertainty penalty,
$\Gamma_{\mathrm{sym}}^\ell$ a fixed symbolic discrepancy,
and $\kappa_{\mathrm{ale}}^\ell$ a fixed signed aleatoric offset.
The mixture is
\begin{equation}
\mathcal{T}^{\mathrm{CS}}_\rho Q
=\sum_{\ell\in\mathcal{M}}\rho(\ell)\mathcal{T}^{\mathrm{CS}}_\ell Q.
\label{eq:csbapr_operator}
\end{equation}

\begin{theorem}[Frozen-operator contraction]
\label{thm:csbapr_contraction}
For stochastic transition kernels, a fixed probability vector $\rho$, fixed $p_\theta,\alpha$, and fixed bounded penalty fields, the operator \eqref{eq:csbapr_operator}, built from the soft value \eqref{eq:soft_value} and mode backups \eqref{eq:csbapr_mode}, is a $\gamma$-contraction on bounded value functions in the supremum norm, with a unique fixed point.
\end{theorem}
\begin{proof}[Proof sketch]
The fixed entropy and penalties cancel between backups.
$|V^{Q_1}(x)-V^{Q_2}(x)|\leq\|Q_1-Q_2\|_\infty$,
and averaging with probability kernels and $\rho$ preserves the bound.
Appendix~\ref{app:proof_contraction} gives the full argument.
\end{proof}

\begin{corollary}[Contraction after belief update]
\label{cor:belief_update}
For nonnegative mode likelihoods $\mathscr{L}(\ell)$ and
$Z_\rho=\sum_\ell\rho(\ell)\mathscr{L}(\ell)>0$,
the posterior
$\rho^+(\ell)=\rho(\ell)\mathscr{L}(\ell)/Z_\rho$
defines another frozen $\gamma$-contractive operator.
\end{corollary}
This is a per-operator property; changing the actor or belief between updates changes the operator and possibly its fixed point. Appendix~\ref{app:proof_belief} proves the probability normalization.

\begin{theorem}[Value-dependent penalty counterexample]
\label{thm:frozen_sym_necessary}
On a one-state, one-action MDP with $0<\gamma<1$, let
$\Gamma_{\mathrm{sym}}(Q)=C_{\mathrm{sym}}-\nu Q$,
where $C_{\mathrm{sym}}$ is a constant and $\nu\geq0$ is coupling sensitivity.
The affine backup has Lipschitz factor
$\gamma(1+\lambda_{\mathrm{sym}}\nu)$.
It fails to contract when this factor is at least one; below that threshold it remains contractive.
\end{theorem}
A constant can make the discrepancy nonnegative on the bounded comparison interval. Appendix~\ref{app:proof_feedback} gives the exact calculation. Freezing is a sufficient way to recover the $\gamma$ template; the counterexample does not prove that every non-frozen design fails.

The abstract operator retains separate frozen penalties, following BAPR's proof template. In the shared-critic implementation, epistemic spread is penalized in the critic and actor losses; the symbolic term is a detached batch average of \eqref{eq:projected_loss}. Appendix~\ref{app:operator_detail} distinguishes these update equations from the abstract operator.

\subsection{Boundary errors and whole-policy regularity}
\label{sec:composition}

\begin{proposition}[Derivative-error propagation]
\label{prop:composition}
Under \Cref{assump:reference,assump:smooth,assump:feature_lip}, for $x\in\Omega$,
\begin{equation}
\|\mathrm{D}\pi(x)-\mathrm{D}\pi_{\mathrm{ref}}(x)\|
\leq\varepsilon+L_{\mathrm{err}}\|x-x_0\|,
\qquad L_{\mathrm{err}}=L_\pi+M.
\label{eq:deriv_error_propagation}
\end{equation}
\end{proposition}
\begin{proof}
Insert both derivatives at $x_0$ and apply the triangle inequality and the two derivative Lipschitz bounds. Appendix~\ref{app:component_lemma} spells out the three terms.
\end{proof}
The smoothness premise covers the OOD segment. This removes the need to assume OOD Jacobian alignment, but it does not remove all OOD-domain assumptions. Applying a constant worst-case segment bound gives the conservative expression
$\delta+\varepsilon\|d\|+L_{\mathrm{err}}\|d\|^2$;
integrating the increasing envelope gives the tighter factor $1/2$ below.

\subsection{Policy-output extrapolation bound}
\label{sec:main_thm}

All maps in this subsection are continuously differentiable on a neighborhood of the relevant closed segment, and all error-bound constants are nonnegative. The arguments use the mean value theorem (MVT) and integration of a derivative envelope.

\begin{lemma}[Scalar linear error growth]
\label{lem:error_linear}
For scalar policy and reference on an interval $[b_-,b_+]$ with $b_-\leq b_+$, if
$|\pi'(x)-\pi_{\mathrm{ref}}'(x)|\leq C_{\mathrm{der}}$
throughout the interval, then the change in $\pi-\pi_{\mathrm{ref}}$ between endpoints is at most
$C_{\mathrm{der}}(b_+-b_-)$.
Here $b_-,b_+$ are interval endpoints and $C_{\mathrm{der}}$ is a uniform derivative-error bound.
\end{lemma}
Appendix~\ref{app:proof_linear} proves the scalar and vector-valued versions.

\begin{theorem}[Scalar quadratic error bound]
\label{thm:ood_1d}
If $|\pi(b_-)-\pi_{\mathrm{ref}}(b_-)|\leq\delta$ and
$|\pi'(x)-\pi_{\mathrm{ref}}'(x)|\leq\varepsilon+
L_{\mathrm{err}}(x-b_-)$ on $[b_-,b_+]$, then
$|\pi(b_+)-\pi_{\mathrm{ref}}(b_+)|
\leq\delta+\varepsilon(b_+-b_-)+
L_{\mathrm{err}}(b_+-b_-)^2/2$.
\end{theorem}
The full proof is in Appendix~\ref{app:proof_scalar}.

\begin{theorem}[Vector-valued linear error bound]
\label{thm:ood_linear_nD}
If the base error is at most $\delta$ and
$\|\mathrm{D}\pi(x)-\mathrm{D}\pi_{\mathrm{ref}}(x)\|
\leq C_{\mathrm{der}}$ on the segment, then
$\|\pi(x_{\mathrm{ood}})-\pi_{\mathrm{ref}}(x_{\mathrm{ood}})\|
\leq\delta+C_{\mathrm{der}}\|d\|$.
\end{theorem}
This is the result used by the uniform symbolic pipelines; see Appendix~\ref{app:proof_linear}.

\begin{lemma}[Path-envelope bound]
\label{lem:path_fencing}
Let $\upsilon\in[0,1]$ parameterize the segment
$\zeta(\upsilon)=x_0+\upsilon d$, and set
$e_{\mathrm{path}}(\upsilon)=\pi(\zeta(\upsilon))-
\pi_{\mathrm{ref}}(\zeta(\upsilon))$.
If $\|e_{\mathrm{path}}(0)\|\leq\delta$ and
$\|e_{\mathrm{path}}'(\upsilon)\|\leq
\varepsilon\|d\|+L_{\mathrm{err}}\|d\|^2\upsilon$, then
$\|e_{\mathrm{path}}(1)\|\leq
\delta+\varepsilon\|d\|+L_{\mathrm{err}}\|d\|^2/2$.
\end{lemma}
Its integration proof appears in Appendix~\ref{app:proof_path}.

\begin{theorem}[Quadratic OOD bound]
\label{thm:ood_quad_nD}
Under \Cref{assump:reference,assump:smooth,assump:feature_lip},
\begin{equation}
\|\pi(x_{\mathrm{ood}})-\pi_{\mathrm{ref}}(x_{\mathrm{ood}})\|
\leq\delta+\varepsilon\|d\|+
\frac{L_{\mathrm{err}}}{2}\|d\|^2.
\label{eq:ood_bound}
\end{equation}
\end{theorem}
\begin{proof}[Proof sketch]
Use \eqref{eq:deriv_error_propagation} along $\zeta(\upsilon)$, apply the chain rule to $e_{\mathrm{path}}$, and invoke \Cref{lem:path_fencing}. Appendix~\ref{app:proof_main} gives the complete proof.
\end{proof}

The three terms describe boundary value error, boundary derivative mismatch, and subsequent derivative drift. They are conditional properties of the compared maps. An affine action map with an affine reference has $L_{\mathrm{err}}=0$; a linear feature-space head inside the implemented actor does not establish that case. SAC return optimization alone also does not certify $\delta$ against an arbitrary reference.

\subsection{Implemented training loop}
\label{sec:algorithm}

\Cref{alg:csbapr} summarizes the primary pipeline. The replay buffer $\mathcal{B}$ supplies sampling distribution $\mathcal{D}$; $N_{\mathrm{upd}}$ is the number of gradient updates per collection step. The optional symbolic loss coefficients are $\lambda_{\mathrm{sym}}$ for the critic target and $\lambda_{\mathrm{JC}}$ for the actor. The primary transit configuration sets both to zero.

\begin{algorithm}[H]
\caption{CS-BAPR training with optional symbolic diagnostics}
\label{alg:csbapr}
\begin{algorithmic}[1]
\STATE Initialize actor parameters $\theta$, critics $\{\phi_k\}$, targets $\{\bar\phi_k\}$, replay buffer $\mathcal{B}$, and BOCD belief $\rho$.
\IF{symbolic diagnostics are enabled}
    \STATE Collect exploration transitions; fit $f_{\mathrm{sym}}$; optionally select across environments; freeze its coefficients.
    \STATE For symbolic penalties, construct and freeze the action-valued reference $\pi_{\mathrm{sym}}=\mathcal C(f_{\mathrm{sym}})$.
\ENDIF
\FOR{each collection step}
    \STATE Collect transitions; append to $\mathcal{B}$ and the recent-rollout tail.
    \STATE Compute surprise $\chi_t$ and update $\rho$.
    \FOR{each of $N_{\mathrm{upd}}$ gradient updates}
        \STATE Sample a replay batch; apply reward normalization and entropy floor.
        \STATE Compute $\lambda_w$ and warmup-adjusted $\beta_{\mathrm{eff}}$.
        \STATE Freeze target values, actor sample, belief, and symbolic discrepancy for the critic target.
        \STATE Compute critic gradients with \eqref{eq:critic_loss} and the minimum-target setting; clip and update critics.
        \STATE Compute actor gradients with \eqref{eq:actor_loss} and the BC-removal setting.
        \STATE Clip actor gradients; update the actor; softly update target networks.
    \ENDFOR
\ENDFOR
\end{algorithmic}
\end{algorithm}
The target and loss definitions are in Appendix~\ref{app:operator_detail}. Disabling symbolic diagnostics retains this training recipe and the fixed-backup contraction template.

\section{Experiments}
\label{sec:experiments}

We evaluate the transit-focused training recipe, with optional symbolic losses disabled. ID operation uses nominal demand. The questions are whether stabilization improves the BAPR implementation, whether the actor-head choice changes OOD performance, and how the behavior transfers to different disturbance axes.

\subsection{Transit benchmark and evaluation protocol}

\paragraph{Environment.}
\texttt{MultiLineEnv} is an event-driven holding simulator for 12 urban bus lines, using SUMO-calibrated headway and dwell dynamics. The single-agent fleet formulation follows the bus-control setting of Zhang~\cite{Zhang2025SingleAgentBus}. At station arrivals, the normalized action $a\in[-1,1]$ is converted to hold duration $h_{\mathrm{hold}}=30(a+1)$ s, with a maximum of 60 s. This keeps the replay action in normalized coordinates and the physical hold time in seconds. The 15-dimensional state includes headways, load, boarding rate, speed, delay, previous action, and line/station/time descriptors. Poisson passenger arrivals use line-specific nominal rates.

\paragraph{Approximate linearity diagnostic.}
Headway-based control~\cite{daganzo2009headway} and optimal linear schedule-reliability control~\cite{xuan2011dynamic} motivate feedback that is locally linear in timing deviations. The recorded diagnostic for a strong learned controller uses 47,247 decisions over five ID episodes: a 50/50 train/test split gives a held-out coefficient of determination $R^2=0.70$ for an ordinary-least-squares fit from normalized state to hold time. The recorded ID clamp rate is zero. This diagnostic suggests approximate local linearity; it does not establish a globally linear policy or prove that optimal holding scales with demand.

\paragraph{OOD protocols.}
The demand multiplier $m$ scales the nominal Poisson arrival rate:
\begin{itemize}
    \item \textbf{Static:} $m\in\{1,2,5,10,20,50\}$ for the full 18,000 s episode.
    \item \textbf{Abrupt:} $m=1$ before $t=3600$ s, followed by one switch to $m\in\{5,10,20,50\}$.
    \item \textbf{Oscillating:} piecewise-constant within-episode schedules. \texttt{commuter\_day} uses multipliers $(1,10,2,20,5,50)$ at times $(0,3600,7200,10800,14400,16200)$ s. \texttt{square\_wave\_20x} alternates 1 and 20 every 1800 s, starting at 1, for nine switches. \texttt{escalating\_10x} uses $(1,2,5,15,50)$ at $(0,3600,7200,10800,14400)$ s; the name is a configuration identifier.
\end{itemize}

\paragraph{Configurations.}
The six reported transit configurations are \texttt{bapr-pristine} (ReLU MLP, six settings disabled), \texttt{bapr} (historical legacy preset), \texttt{csbapr-no-nau} (ReLU MLP with stabilization), \texttt{csbapr} (mixed arithmetic head with stabilization), \texttt{csbapr-tanh} (tanh-hidden MLP with stabilization), and \texttt{csbapr-kan} (KAN with stabilization). The BC coefficient follows \Cref{sec:training_fixes}. Every transit row uses $N_{\mathrm{seed}}=10$ seeds and 500 training episodes per seed, with the selected ID checkpoint evaluated under the OOD protocols. Selection is based on the best ID score rather than an OOD score. Tables report medians across seeds except \Cref{tab:main}, which reports mean and standard deviation. Return-unit suffixes K and M mean $10^3$ and $10^6$; higher return is better. Appendix~\ref{app:hyperparams} documents defaults and overrides.

\subsection{ID convergence}
\label{sec:id_results}

\Cref{tab:main} reports ID total return across the 12 lines. The arithmetic actor is comparable to the stabilized ReLU variant, and KAN has the highest mean. Different aggregation explains why these means differ from the medians at $m=1$ in \Cref{tab:ood_sweep}.

\begin{table}[htbp]
\centering
\caption{ID total return across the 12 lines at $m=1$. Higher is better. Mean $\pm$ standard deviation, $N_{\mathrm{seed}}{=}10$ seeds per row.}
\label{tab:main}
\begin{tabular}{lcr}
\toprule
Method & $N_{\mathrm{seed}}$ seeds & ID total return (K) \\
\midrule
\texttt{bapr-pristine} (ReLU $+$ no fixes)        & 10 & $-850 \pm 37$ \\
\texttt{bapr} (historical legacy)              & 10 & $-793 \pm 34$ \\
\texttt{csbapr-no-nau} (ReLU-MLP $+$ 6 fixes)     & 10 & $-780 \pm 55$ \\
\texttt{csbapr} (NAU $+$ 6 fixes)                 & 10 & $-772 \pm 36$ \\
\texttt{csbapr-tanh} (smooth-MLP $+$ 6 fixes)     & 10 & $-817 \pm 54$ \\
\texttt{csbapr-kan} (KAN $+$ 6 fixes)             & 10 & $-759 \pm 35$ \\
\bottomrule
\end{tabular}
\end{table}

\subsection{Demand OOD and stabilization ablation}
\label{sec:ood_results}

\Cref{tab:ood_sweep,tab:abrupt,tab:oscillating} report the three demand protocols, and \Cref{tab:crash} reports severe seed outcomes. The stabilization comparison holds the ReLU architecture fixed: relative to \texttt{bapr-pristine}, \texttt{csbapr-no-nau} reduces absolute static return loss by $10.6\%$, $11.9\%$, $19.0\%$, $15.1\%$, and $7.2\%$ at $m=1,5,10,20,50$, respectively. Relative improvement here means the return difference divided by the absolute baseline return.

The arithmetic variant's corresponding gains over \texttt{bapr-pristine} are $10.3\%$, $14.1\%$, $16.2\%$, $13.3\%$, and $8.7\%$. Its gain over the stabilized ReLU variant at $m=5$ is about $2.4\%$. The head comparison includes different architectures and head regularizers, rather than an isolated intervention on derivative curvature.

\begin{table}[htbp]
\centering
\scriptsize
\caption{Static OOD demand sweep, total episode return (K, higher is better). Median episode return over $N_{\mathrm{seed}}{=}10$ seeds per row. $m=1$ is ID. \texttt{bapr-pristine} disables all six stabilization settings; \texttt{bapr} is a historical legacy configuration. \textbf{Bold}: best in column.}
\label{tab:ood_sweep}
\begin{tabular}{lrrrrrr}
\toprule
Method & $m{=}1$ & $m{=}2$ & $m{=}5$ & $m{=}10$ & $m{=}20$ & $m{=}50$ \\
\midrule
\texttt{bapr-pristine} (ReLU, no fixes)         & $-861$ & $-862$ & $-939$ & $-1134$ & $-1267$ & $-1416$ \\
\texttt{bapr} (historical legacy)            & $-791$ & $-799$ & $-878$ & $-1025$ & $-1196$ & $-1428$ \\
\texttt{csbapr-no-nau} (ReLU-MLP$+$6 fixes)     & $-770$ & $-776$ & $-827$ & $\mathbf{-919}$ & $\mathbf{-1076}$ & $-1314$ \\
\texttt{csbapr} (NAU$+$6 fixes)                 & $-772$ & $-778$ & $\mathbf{-807}$ & $-950$ & $-1098$ & $-1293$ \\
\texttt{csbapr-tanh} (smooth-MLP$+$6 fixes)     & $-786$ & $-824$ & $-897$ & $-986$ & $-1087$ & $\mathbf{-1239}$ \\
\texttt{csbapr-kan} (KAN$+$6 fixes)             & $\mathbf{-756}$ & $\mathbf{-762}$ & $-820$ & $-971$ & $-1121$ & $-1335$ \\
\bottomrule
\end{tabular}
\end{table}

\begin{table}[htbp]
\centering
\scriptsize
\caption{Abrupt-burst protocol: $m = 1$ for $t \in [0, 3600)$ s, then abrupt switch to the specified multiplier. Total episode return (K), median over $N_{\mathrm{seed}}{=}10$ seeds.}
\label{tab:abrupt}
\begin{tabular}{lrrrr}
\toprule
Method & burst $5{\times}$ & burst $10{\times}$ & burst $20{\times}$ & burst $50{\times}$ \\
\midrule
\texttt{bapr-pristine}    & $-909$  & $-1021$ & $-1178$ & $-1354$ \\
\texttt{bapr} (legacy)    & $-836$  & $-963$  & $-1104$ & $-1316$ \\
\texttt{csbapr-no-nau}    & $-792$ & $\mathbf{-858}$  & $\mathbf{-982}$  & $-1205$ \\
\texttt{csbapr} (NAU)     & $-792$ & $-881$ & $-1000$ & $-1196$ \\
\texttt{csbapr-tanh}      & $-847$  & $-932$  & $-1017$ & $\mathbf{-1152}$ \\
\texttt{csbapr-kan}       & $\mathbf{-790}$  & $-909$  & $-1054$ & $-1224$ \\
\bottomrule
\end{tabular}
\end{table}

\begin{table}[htbp]
\centering
\scriptsize
\caption{Oscillating OOD protocol: within-episode piecewise-constant demand schedule. Total episode return (K), median over $N_{\mathrm{seed}}{=}10$ seeds. Schedules: \texttt{commuter\_day} has 5 switches with peak $50{\times}$; \texttt{square\_wave\_20x} has 9 switches with peak $20{\times}$; \texttt{escalating\_10x} has 4 switches with peak $50{\times}$.}
\label{tab:oscillating}
\begin{tabular}{lrrr}
\toprule
Method & \texttt{commuter\_day} & \texttt{square\_wave\_20x} & \texttt{escalating\_10x} \\
\midrule
\texttt{bapr-pristine}    & $-998$  & $-1111$ & $-956$ \\
\texttt{bapr} (legacy)    & $-954$  & $-1045$ & $-920$ \\
\texttt{csbapr-no-nau}    & $\mathbf{-859}$ & $\mathbf{-921}$ & $\mathbf{-845}$ \\
\texttt{csbapr} (NAU)     & $-865$  & $-927$  & $-853$ \\
\texttt{csbapr-tanh}      & $-929$  & $-990$  & $-908$ \\
\texttt{csbapr-kan}       & $-914$  & $-972$  & $-912$ \\
\bottomrule
\end{tabular}
\end{table}

\begin{table}[htbp]
\centering
\scriptsize
\caption{Crash rate at $m{=}20{\times}$ across multiple severity thresholds. Each cell reports the fraction of seeds whose return at $m{=}20{\times}$ falls below the threshold (lower is better). All six methods use $N_{\mathrm{seed}}{=}10$ seeds. Thresholds show both typical poor performance and severe tail outcomes.}
\label{tab:crash}
\setlength{\tabcolsep}{5pt}
\begin{tabular}{lcccccc}
\toprule
Method & $N_{\mathrm{seed}}$ & $<{-}1{,}200$K & $<{-}1{,}500$K & $<{-}1{,}800$K & $<{-}2{,}000$K & Worst seed (K) \\
\midrule
\texttt{bapr-pristine} (ReLU, no fixes)        & 10 & $70\%$ & $0\%$  & $0\%$  & $0\%$ & $-1{,}422$ \\
\texttt{bapr} (historical legacy)           & 10 & $50\%$ & $0\%$  & $0\%$  & $0\%$ & $-1{,}367$ \\
\texttt{csbapr-no-nau} (ReLU-MLP$+$6 fixes)    & 10 & $40\%$ & $0\%$  & $0\%$  & $0\%$ & $-1{,}309$ \\
\texttt{csbapr} (NAU$+$6 fixes)                & 10 & $\mathbf{20\%}$ & $10\%$ & $10\%$ & $0\%$ & $-1{,}851$ \\
\texttt{csbapr-tanh} (smooth-MLP$+$6 fixes)    & 10 & $30\%$ & $0\%$  & $0\%$  & $0\%$ & $\mathbf{-1{,}257}$ \\
\texttt{csbapr-kan} (KAN$+$6 fixes)            & 10 & $30\%$ & $10\%$ & $0\%$  & $0\%$ & $-1{,}551$ \\
\bottomrule
\end{tabular}
\end{table}

No architecture dominates. KAN has the best static medians at $m=1,2$ and the best burst-$5\times$ value. The arithmetic variant is best at static $m=5$; the stabilized ReLU variant is best at static and burst $m=10,20$ and all three oscillating schedules; the tanh variant is best at static and burst $m=50$. These are rankings of the reported medians, without a significance claim.

The crash thresholds distinguish typical performance from severe tails. At $m=20$, the arithmetic variant has one seed at $-1{,}851$K, KAN has one at $-1{,}551$K, and no configuration has a seed below $-2{,}000$K. The arithmetic head therefore does not eliminate severe outcomes. The reported returns cannot establish that an outlier was caused by a particular weight constraint or derivative constant.

\subsection{Return degradation diagnostic}
\label{sec:bound_alignment}

\Cref{fig:bound} plots
$\Delta J(m)=J(1)-J(m)$,
where $J(m)$ is the static-table median episode return at demand multiplier $m$. Thus the plot is a difference of medians, not a median of paired seed differences. It visualizes degradation of the same reported configurations and data in \Cref{tab:ood_sweep}.

The demand axis $m$ is not the state displacement $d$ in \eqref{eq:ood_bound}, and return has different units from action-output error. We therefore compare the observed degradation directly, without overlaying an unmeasured policy-output certificate.

\begin{figure}[H]
\centering
\includegraphics[width=0.92\linewidth]{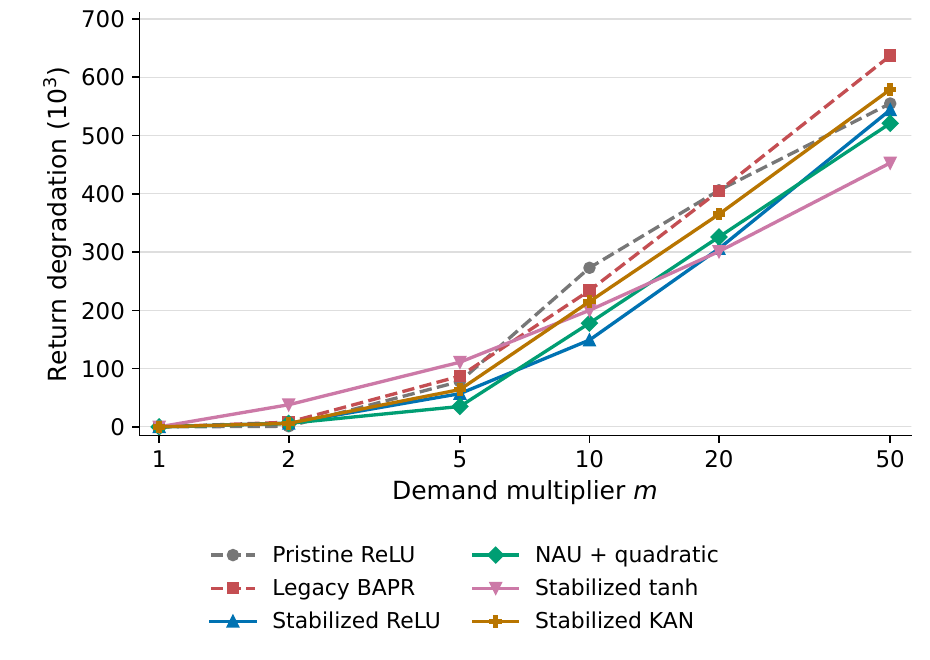}
\caption{Static-demand return degradation relative to each configuration's ID median. Every point is computed from \Cref{tab:ood_sweep}, with $N_{\mathrm{seed}}=10$ for each underlying median. Lower degradation is better relative to that configuration's own ID level. The figure reports a difference of aggregate medians and has no seed-level uncertainty bands.}
\label{fig:bound}
\end{figure}

\FloatBarrier

\subsection{Auxiliary benchmarks}
\label{sec:aux_benchmarks}

\paragraph{Exploratory LQR amplitude sweep.}
A point-mass double integrator with quadratic cost is a simple reference setting. In the unsaturated discrete linear model, $x_{t+1}=Ax_t+B_{\mathrm{in}}a_t$ and the optimal feedback is $a^*=-K_{\mathrm{fb}}x$, where $A$ is the state-dynamics matrix, $B_{\mathrm{in}}$ the input matrix, and $K_{\mathrm{fb}}$ the feedback gain determined by dynamics and cost. An identified $A$ alone is not that feedback gain. Initial-state radius is multiplied by $m_{\mathrm{amp}}$, and the reported historical cells are shown in \Cref{tab:lqr}.

\begin{table}[htbp]
\centering
\scriptsize
\caption{Exploratory LQR initial-amplitude sweep, median unaugmented episode return over $N_{\mathrm{seed}}=3$ seeds. Higher is better. K and M denote $10^3$ and $10^6$ return units; -- denotes an unreported cell. The symbolic-enabled configurations were labeled with \texttt{SINDY\_WITH\_CONTROL=1}; the implementation limitation is discussed in \Cref{sec:limitations}. \textbf{Bold}: best reported value in the column.}
\label{tab:lqr}
\begin{tabular}{lrrrrrr}
\toprule
Method & $m_{\mathrm{amp}}{=}1$ & $m_{\mathrm{amp}}{=}2$ & $m_{\mathrm{amp}}{=}5$ & $m_{\mathrm{amp}}{=}10$ & $m_{\mathrm{amp}}{=}20$ & $m_{\mathrm{amp}}{=}50$ \\
\midrule
\texttt{bapr} (no SINDy, no NAU)         & $\mathbf{-168}$  & --      & $-138$K & $-1.28$M & --      & $\mathbf{-9.40}$M \\
\texttt{csbapr-no-sindy} (NAU, no SINDy) & $-1{,}164$  & --      & $-105$K & $-2.26$M & --      & $-11.85$M \\
\texttt{csbapr-no-nau} (MLP$+$SINDy)     & $-2{,}301$  & --      & $-18$K  & $-0.94$M & --      & $-9.84$M \\
\texttt{csbapr} (NAU$+$SINDy, full)      & $-2{,}916$  & --      & $\mathbf{-10}$K  & $\mathbf{-0.90}$M & --      & $-11.45$M \\
\bottomrule
\end{tabular}
\end{table}

Within the reported later sweep, the symbolic-enabled arithmetic configuration improves absolute return loss over \texttt{csbapr-no-sindy} by about $90.5\%$ at $m_{\mathrm{amp}}=5$ ($-10$K versus $-105$K) and $60.2\%$ at $m_{\mathrm{amp}}=10$ ($-0.90$M versus $-2.26$M). The earlier state-only configuration reported $-381$K and $-2.65$M at those amplitudes; comparing the two versions gives $97.4\%$ and $66.0\%$, a different comparison from the no-SINDy ablation. At ID, BAPR is best, and at $m_{\mathrm{amp}}=50$ its $-9.40$M is also the highest reported return. The symbolic-enabled results are retained as exploratory configuration comparisons; the action-input and JC limitations are stated in \Cref{sec:limitations}.

\paragraph{Hopper-v4 wind.}
A constant horizontal force acts on the torso:
\[
F_{\mathrm{wind}}=\omega m_{\mathrm{torso}}g,
\]
where $\omega\in\{0,1,2,5,10,20\}$ is dimensionless wind amplitude, $m_{\mathrm{torso}}$ is torso mass, and $g$ is gravitational acceleration. Training uses $\omega=0$; evaluation sweeps the full range. \Cref{tab:hopper} reports the medians.

\begin{table}[htbp]
\centering
\scriptsize
\caption{Hopper-v4 wind sweep, median episode return over $N_{\mathrm{seed}}=3$ seeds except the legacy comparison ($N_{\mathrm{seed}}=2$, interrupted runs excluded). The pristine row disables the six stabilization settings; the legacy row is retained as a historical comparison. Wind amplitude $\omega=0$ is ID. \textbf{Bold}: best in the column.}
\label{tab:hopper}
\begin{tabular}{lrrrrrr}
\toprule
Method & ID ($\omega{=}0$) & $\omega{=}1$ & $\omega{=}2$ & $\omega{=}5$ & $\omega{=}10$ & $\omega{=}20$ \\
\midrule
\texttt{bapr-pristine} (ReLU, no fixes)              & $635$  & $294$           & $\mathbf{276}$           & $85$            & $127$           & $197$           \\
\texttt{bapr}/\texttt{csbapr-relu} (legacy, $N_{\mathrm{seed}}{=}2$)  & $995$  & $216$           & $221$           & $144$           & $99$            & $107$           \\
\texttt{csbapr-tanh} (smooth-MLP$+$6 fixes)           & $975$  & $\mathbf{330}$  & $251$           & $203$           & $239$           & $56$            \\
\texttt{csbapr-no-nau} (ReLU-MLP$+$6 fixes)           & $\mathbf{1{,}008}$ & $317$           & $270$  & $\mathbf{255}$  & $\mathbf{256}$  & $\mathbf{224}$  \\
\texttt{csbapr} (NAU$+$6 fixes)                       & $990$              & $289$           & $175$           & $62$            & $81$            & $74$            \\
\bottomrule
\end{tabular}
\end{table}

The stabilized ReLU variant has the highest ID median and is best at $\omega=5,10,20$. The pristine variant is best at $\omega=2$, and tanh is best at $\omega=1$. The arithmetic variant falls from 990 at ID to 62 at $\omega=5$ and 74 at $\omega=20$, so its transit ranking does not transfer to this disturbance axis. Holding ReLU fixed, stabilization raises return from 635 to 1008 at ID, from 85 to 255 at $\omega=5$, and from 197 to 224 at $\omega=20$. This supports a useful training recipe in these evaluated cases, with task-dependent architectural behavior.

\FloatBarrier

\section{Discussion and conclusion}
\label{sec:conclusion}

The strongest empirical finding is that the stabilization recipe improves the evaluated BAPR implementation on the transit demand protocols. Changing the actor head provides smaller, regime-dependent differences: arithmetic is best at static $5\times$ demand, KAN at nominal and near-nominal demand, stabilized ReLU at several mid-to-high demand and switching conditions, and tanh at the $50\times$ tail.

The theoretical finding is the conditional action-output bound \eqref{eq:ood_bound}, whose constants include both policy and reference regularity. The layer calculations in \Cref{thm:nau_const,thm:nmu_bias,thm:relu_impossible} explain one source of curvature control, while the composition analysis identifies the additional assumptions needed for a complete policy certificate. The frozen Bellman analysis retains the predecessor's fixed-backup contraction template. The proof map in Appendix~\ref{app:lean_overview} connects each result to the supplied proof rather than to an unavailable formal artifact.

\subsection{Limitations}
\label{sec:limitations}

\paragraph{Whole-policy certification and performance.}
The historical arithmetic actor has LeakyReLU feature layers, a quadratic correction, and tanh squashing. It has not been certified to satisfy \Cref{assump:feature_lip} across the OOD segments; its earlier curvature surrogate was not a valid $L_\pi$. The revised \texttt{smooth-nau} variant uses tanh features and computes a whole deterministic-policy bound, including the quadratic projection and action squash, as detailed in Appendix~\ref{app:leff}. It leaves the historical actor available and reports no finite certificate for that actor or a hard mean clamp. The certificate covers normalized action outputs, not sampled-action distributions, episode return, or deep actor-critic convergence. It can be conservative. The reported tables do not evaluate the new variant.

\paragraph{Symbolic and invariance components.}
The historical symbolic implementation omitted actions and used the projected loss \eqref{eq:projected_loss}. The revised code supplies actions to fitting and prediction, constructs an LQR feedback reference, and implements \eqref{eq:jac_loss}. These repairs are implementation changes, not retrospective validation of \Cref{tab:lqr}; new performance experiments are needed. Other tasks require their own action-valued control construction. Prediction-error variance can prefer a consistently inaccurate model; the selector neither proves derivative accuracy nor establishes SCM causality. IRM filtering is inactive in the reported primary runs.

\paragraph{Experimental scope and accounting.}
The 12-line transit simulator is one application setting; LQR and Hopper have three seeds per principal configuration. The legacy Hopper comparison has two completed seeds with interrupted runs excluded, and does not match the pristine baseline. The historical transit \texttt{bapr} row predates later defaults; the current \texttt{bapr} entrypoint inherits those defaults and does not by itself reproduce that older configuration. Actor weight decay and regularizer strengths also belong to the compared recipe, so the pristine comparison is not a six-factor causal ablation. Standard SAC, domain-randomized training, and robust adversarial RL (RARL) are not included in the primary result tables. The comparisons establish improvement over the evaluated predecessor configurations, with no universal claim for the stabilization package or head.

\paragraph{Reproducibility and proof artifacts.}
The manuscript-level comma-separated values (CSV) file transcribes aggregate table values. It does not replace original seed-level logs, checkpoint-selection records, or traces supporting the linear-fit diagnostic; those historical artifacts are absent from the retained package following storage cleanup. The diagnostics and statistical comparisons cannot be independently reconstructed from this CSV alone. This repository snapshot supplies mathematical proofs but no executable Lean project for the CS-BAPR results. Formalization tags in earlier drafts were not a machine-checking certificate.

The next empirical step is a fresh comparison of the revised smooth actor and action-valued feedback alignment, retaining per-seed returns and checkpoint-selection records. Boundary value and operator-norm derivative errors in \eqref{eq:ood_bound} are measured at the stated boundary point, not inferred from replay-average losses or sample counts. Larger auxiliary seed sweeps and additional standard baselines would then assess transfer beyond transit demand scaling.

\section*{Code and Data Availability}
Code is available at \url{https://github.com/erzhu419/CS-BAPR}. Aggregate source values for \Cref{tab:main,tab:ood_sweep,tab:abrupt,tab:oscillating,tab:crash,tab:lqr,tab:hopper} are transcribed in \path{paper/source_data/reported_tables.csv}. The figure is regenerated from its static-sweep rows. Original seed-level evidence is not included in the retained manuscript package.

\section*{Acknowledgments}
This work was supported by the National Natural Science Foundation of China (Grant No. 72371251), the Natural Science Foundation for Distinguished Young Scholars of Hunan Province (Grant No. 2024JJ2080), and the Key Research and Development Program of Hunan Province of China (Grant No. 2024JK2007).


\appendix

\section{Proof map}
\label{app:lean_overview}

The appendix supplies mathematical proofs and implementation equations. The main text invokes each proof at the result it supports. The dependency map is:
\begin{itemize}
\item \Cref{lem:error_linear,thm:ood_linear_nD}: Appendix~\ref{app:proof_linear}.
\item \Cref{thm:ood_1d,lem:path_fencing,thm:ood_quad_nD}: Appendices~\ref{app:proof_scalar}, \ref{app:proof_path}, and \ref{app:proof_main}.
\item \Cref{prop:composition}: Appendix~\ref{app:component_lemma}.
\item \Cref{thm:csbapr_contraction,cor:belief_update,thm:frozen_sym_necessary}: Appendices~\ref{app:proof_contraction}--\ref{app:proof_feedback}.
\item \Cref{thm:sindy_jac,thm:sindy_pipeline}: Appendix~\ref{app:proof_symbolic}.
\item \Cref{thm:irm_causal,thm:irm_pipeline}: Appendix~\ref{app:proof_irm}.
\item The composition inequality \eqref{eq:composition_rule}: Appendix~\ref{app:leff}.
\end{itemize}
These are paper proofs; an executable Lean project is not included in the current artifact.

\section{Extrapolation proofs}

\subsection{Linear derivative-error bounds}
\label{app:proof_linear}

\begin{proof}[Proof of \Cref{lem:error_linear}]
Put \(e_{\mathrm{out}}(x)=\pi(x)-\pi_{\mathrm{ref}}(x)\). By the fundamental theorem of calculus,
\[
|e_{\mathrm{out}}(b_+)-e_{\mathrm{out}}(b_-)|=\left|\int_{b_-}^{b_+}e_{\mathrm{out}}'(x)\,\mathrm d x\right|
\leq\int_{b_-}^{b_+}C_{\mathrm{der}}\,\mathrm d x
=C_{\mathrm{der}}(b_+-b_-).
\]
\end{proof}

\begin{proof}[Proof of \Cref{thm:ood_linear_nD}]
Parametrize the segment by \(\zeta(\upsilon)=x_0+\upsilon d\), where
\(d=x_{\mathrm{ood}}-x_0\) and \(\upsilon\in[0,1]\). For
\(e_{\mathrm{out}}=\pi-\pi_{\mathrm{ref}}\), the chain rule gives
\[
\left\|\frac{\mathrm d}{\mathrm d\upsilon}e_{\mathrm{out}}(\zeta(\upsilon))\right\|
=\|\mathrm{D}e_{\mathrm{out}}(\zeta(\upsilon))d\|
\leq C_{\mathrm{der}}\|d\|.
\]
Integrating and adding \(\|e_{\mathrm{out}}(x_0)\|\leq\delta\) yields
\(\|e_{\mathrm{out}}(x_{\mathrm{ood}})\|\leq\delta+C_{\mathrm{der}}\|d\|\).
\end{proof}

\subsection{Scalar quadratic bound}
\label{app:proof_scalar}

\begin{proof}[Proof of \Cref{thm:ood_1d}]
With \(e_{\mathrm{out}}=\pi-\pi_{\mathrm{ref}}\), integrate the assumed envelope:
\[
|e_{\mathrm{out}}(b_+)|\leq|e_{\mathrm{out}}(b_-)|+\int_{b_-}^{b_+}|e_{\mathrm{out}}'(x)|\,\mathrm d x
\leq\delta+\int_{b_-}^{b_+}
[\varepsilon+L_{\mathrm{err}}(x-b_-)]\,\mathrm d x.
\]
The integral equals
\(\varepsilon(b_+-b_-)+L_{\mathrm{err}}(b_+-b_-)^2/2\).
\end{proof}

\subsection{Path-envelope integration}
\label{app:proof_path}

\begin{proof}[Proof of \Cref{lem:path_fencing}]
The path error is continuously differentiable. Therefore
\[
\begin{split}
\|e_{\mathrm{path}}(1)\|
&\leq\|e_{\mathrm{path}}(0)\|
+\int_0^1\|e_{\mathrm{path}}'(\upsilon)\|\,\mathrm d\upsilon\\
&\leq\delta+\int_0^1
[\varepsilon\|d\|+L_{\mathrm{err}}\|d\|^2\upsilon]\,\mathrm d\upsilon\\
&=\delta+\varepsilon\|d\|+\tfrac12 L_{\mathrm{err}}\|d\|^2.
\end{split}
\]
\end{proof}

\subsection{Derivative-error propagation}
\label{app:component_lemma}

\begin{proof}[Expanded proof of \Cref{prop:composition}]
Insert the boundary derivatives and apply the triangle inequality:
\[
\begin{aligned}
\|\mathrm D\pi(x)-\mathrm D\pi_{\mathrm{ref}}(x)\|
&\leq \|\mathrm D\pi(x)-\mathrm D\pi(x_0)\|
+\|\mathrm D\pi(x_0)-\mathrm D\pi_{\mathrm{ref}}(x_0)\|\\
&\quad+\|\mathrm D\pi_{\mathrm{ref}}(x_0)-\mathrm D\pi_{\mathrm{ref}}(x)\|\\
&\leq L_\pi\|x-x_0\|+\varepsilon+M\|x-x_0\|.
\end{aligned}
\]
This is \eqref{eq:deriv_error_propagation} with
\(L_{\mathrm{err}}=L_\pi+M\).
\end{proof}

\subsection{Whole-policy quadratic bound}
\label{app:proof_main}

\begin{proof}[Proof of \Cref{thm:ood_quad_nD}]
By \Cref{prop:composition}, at every point of the segment,
\[
\|\mathrm{D}\pi(\zeta(\upsilon))
-\mathrm{D}\pi_{\mathrm{ref}}(\zeta(\upsilon))\|
\leq\varepsilon+L_{\mathrm{err}}\upsilon\|d\|.
\]
Multiplication by \(d\) through the chain rule bounds the path-error derivative by
\(\varepsilon\|d\|+L_{\mathrm{err}}\upsilon\|d\|^2\).
The base error is at most \(\delta\). Apply \Cref{lem:path_fencing}.
For \(d=0\), the statement is simply the base-error assumption.
\end{proof}

\section{Bellman and symbolic-reference proofs}

\subsection{Frozen-operator contraction}
\label{app:proof_contraction}

\begin{proof}[Proof of \Cref{thm:csbapr_contraction}]
Let \(d_Q=\|Q_1-Q_2\|_\infty\). The actor and temperature are identical in the two soft values, so their entropy terms cancel:
\[
|V^{Q_1}(x')-V^{Q_2}(x')|
\leq\int_{\mathcal A}p_\theta(a'\mid x')
|Q_1(x',a')-Q_2(x',a')|\,\mathrm d a'\leq d_Q.
\]
Rewards and frozen penalties also cancel. Since each \(P_\ell\) is a probability kernel,
\[
|(\mathcal T^{\mathrm{CS}}_\ell Q_1)(x,a)
-(\mathcal T^{\mathrm{CS}}_\ell Q_2)(x,a)|
\leq\gamma\int_{\mathcal S}d_Q P_\ell(dx'\mid x,a)=\gamma d_Q.
\]
Averaging with the fixed nonnegative weights \(\rho(\ell)\), whose sum is one, and taking the supremum preserves this inequality. Bounded rewards, entropy contributions, and penalties make the operator a self-map on the complete space of bounded measurable value functions. The contraction mapping theorem gives the unique fixed point.
\end{proof}

\subsection{Belief normalization}
\label{app:proof_belief}

\begin{proof}[Proof of \Cref{cor:belief_update}]
Every \(\rho^+(\ell)\) is nonnegative and
\(\sum_\ell\rho^+(\ell)=Z_\rho/Z_\rho=1\).
Appendix~\ref{app:proof_contraction} therefore applies with \(\rho^+\) fixed.
The conclusion concerns this new operator, not convergence of a sequence of changing operators.
\end{proof}

\subsection{Value-dependent feedback}
\label{app:proof_feedback}

\begin{proof}[Proof of \Cref{thm:frozen_sym_necessary}]
In the one-state, one-action example the entropy is zero and the next-state value is \(Q\). Absorbing all fixed terms into a constant \(C_0\), the backup is
\[
\mathcal T Q=C_0+\gamma(1+\lambda_{\mathrm{sym}}\nu)Q.
\]
For distinct scalar values its exact Lipschitz ratio is
\(\gamma(1+\lambda_{\mathrm{sym}}\nu)\).
A sufficiently large \(C_{\mathrm{sym}}\) makes
\(C_{\mathrm{sym}}-\nu Q\geq0\) on any chosen bounded comparison interval.
The affine map contracts precisely when its ratio is below one.
\end{proof}

\subsection{Symbolic-reference pipeline}
\label{app:proof_symbolic}

\begin{proof}[Proof of \Cref{thm:sindy_jac,thm:sindy_pipeline}]
At any point where the two hypotheses hold,
\[
\begin{split}
\|\mathrm D\pi-\mathrm D\pi_{\mathrm{ref}}\|
&\leq\|\mathrm D\pi-\mathrm D\pi_{\mathrm{sym}}\|
+\|\mathrm D\pi_{\mathrm{sym}}-\mathrm D\pi_{\mathrm{ref}}\|\\
&\leq\varepsilon_2+\varepsilon_1.
\end{split}
\]
If this holds throughout the segment, apply \Cref{thm:ood_linear_nD}
with \(C_{\mathrm{der}}=\varepsilon_1+\varepsilon_2\).
If it holds only at \(x_0\), it bounds the boundary mismatch;
\Cref{prop:composition,thm:ood_quad_nD} then require and use the whole-segment regularity assumptions.
Both derivatives in the argument are action-valued, after the stated control construction; a state-valued dynamics output is not interchangeable with that reference.
\end{proof}

\subsection{Conditional invariant filtering}
\label{app:proof_irm}

\begin{proof}[Proof of \Cref{thm:irm_causal}]
For the invariant model, decompose its test-environment derivative error as
\[
\mathrm D\pi_{\mathrm{ref}}^{e_\star}
-\mathrm D\pi_{\mathrm{sym}}^{\mathrm{inv}}
=(\mathrm D\pi_{\mathrm{inv}}
-\mathrm D\pi_{\mathrm{sym}}^{\mathrm{inv}})
+\mathrm D\pi_{\mathrm{var}}^{e_\star}.
\]
Its norm is at most
\(\varepsilon_{\mathrm{inv}}+\varepsilon_{\mathrm{var}}^{e_\star}\).
For the single-environment model insert the training-environment reference:
\[
\begin{split}
\mathrm D\pi_{\mathrm{ref}}^{e_\star}
-\mathrm D\pi_{\mathrm{sym}}^{\mathrm{single}}
={}&[\mathrm D(\pi_{\mathrm{inv}}+\pi_{\mathrm{var}}^{e_0})
-\mathrm D\pi_{\mathrm{sym}}^{\mathrm{single}}]\\
&+\mathrm D\pi_{\mathrm{var}}^{e_\star}
-\mathrm D\pi_{\mathrm{var}}^{e_0}.
\end{split}
\]
The triangle inequality gives the stated second bound. These are upper bounds on two distinct errors. Their difference does not establish a lower bound on the difference of the actual errors.
\end{proof}

\begin{proof}[Proof of \Cref{thm:irm_pipeline}]
In environment \(e_\star\), add the uniform derivative bounds
\(\varepsilon_{\mathrm{irm}}\) and \(\varepsilon_2\) by the same triangle inequality.
Apply \Cref{thm:ood_linear_nD} with the stated base error.
\end{proof}

\section{Implemented update equations}
\label{app:operator_detail}

For a replay transition \((x,a,r,x',d_{\mathrm{term}})\),
\(d_{\mathrm{term}}\in\{0,1\}\) is the terminal indicator and \(\widetilde r\)
is the normalized reward when normalization is enabled. Sample
\(a'\sim p_\theta(\cdot\mid x')\).
Let \(n_{\bar\phi_k}\geq0\) be the parameter norm returned by
\texttt{compute\_reg\_norm}: the sum of absolute target-critic weights and
hidden biases, excluding the output bias. The code adds this norm offset;
it is not a directly measured aleatoric uncertainty.
For the minimum-ensemble setting, the detached target for critic \(k\) is
\begin{equation}
\begin{split}
y_k=\widetilde r+\gamma(1-d_{\mathrm{term}})\big[
&Q_{\min}(x',a')-\alpha\log p_\theta(a'\mid x')\\
&+\lambda_{\mathrm{reg}}n_{\bar\phi_k}
-\lambda_{\mathrm{sym}}\widehat{\mathcal L}_{\mathrm{JC}}\big],
\end{split}
\label{eq:critic_target}
\end{equation}
where \(\lambda_{\mathrm{reg}}\) is the norm-offset coefficient and
\(\widehat{\mathcal L}_{\mathrm{JC}}\) is the batch estimator of
\eqref{eq:jac_loss} in the revised implementation. The historical symbolic
target instead used \eqref{eq:projected_loss}. For the per-critic-target setting,
\(Q_{\min}\) is replaced by \(Q_{\bar\phi_k}\).
The symbolic term is omitted when the symbolic model is unavailable or disabled.

Define the ensemble mean and sample standard deviation by
\[
\bar Q(x,a)=\frac{1}{N_{\mathrm{crit}}}\sum_{k=1}^{N_{\mathrm{crit}}}Q_{\phi_k}(x,a),
\qquad
\sigma_Q(x,a)=
\left[\frac{\sum_k(Q_{\phi_k}(x,a)-\bar Q(x,a))^2}
{N_{\mathrm{crit}}-1}\right]^{1/2}.
\]
The shared-critic loss is
\begin{equation}
\mathcal L_Q=
\mathbb E_{\mathrm{batch}}\!\left[
\frac{1}{N_{\mathrm{crit}}}\sum_k(Q_{\phi_k}(x,a)-y_k)^2
+\beta_{\mathrm{ood}}\sigma_Q(x,a)\right],
\label{eq:critic_loss}
\end{equation}
where \(\mathbb E_{\mathrm{batch}}\) denotes the replay-batch average and
\(\beta_{\mathrm{ood}}\) is its spread coefficient.
For a fresh action \(\widehat a\sim p_\theta(\cdot\mid x)\), the arithmetic-actor objective is
\begin{equation}
\begin{split}
\mathcal L_\pi={}&
\mathbb E_{\mathrm{batch},\,\widehat a}
[-\bar Q(x,\widehat a)-\beta_{\mathrm{eff}}\sigma_Q(x,\widehat a)
+\alpha\log p_\theta(\widehat a\mid x)
+\frac{\beta_{\mathrm{bc}}}{d_a}\|\widehat a-a\|^2]\\
&+\lambda_{\mathrm{JC}}\widehat{\mathcal L}_{\mathrm{JC}}
+\lambda_{\mathrm{NAU}}\mathcal L_{\mathrm{NAU}}.
\end{split}
\label{eq:actor_loss}
\end{equation}
Here \(\lambda_{\mathrm{JC}}\) and \(\lambda_{\mathrm{NAU}}\) weight the
full action-space consistency term and the NAU weight regularizer. The revised
actor uses the unclipped consistency loss; the historical version clipped
the projected loss at 2. The implemented
BC term is a mean squared error over the \(d_a\) action coordinates.
Other heads omit the NAU regularizer or use their own regularizer.
The primary transit comparisons set \(\lambda_{\mathrm{JC}}\) and
\(\lambda_{\mathrm{sym}}\) to zero. The six-setting comparison does not
activate the optional Q-spread caps or extra curvature penalties.

For the implementation's run-length hypotheses, \(\ell\) takes the values
\(0,\ldots,H_{\max}-1\). At update \(j\), the code uses
\[
\lambda_w=\sum_{\ell=0}^{H_{\max}-1}\rho_j(\ell)\exp(-0.1\ell)
\]
after warmup; it is zero during warmup. Here \(\rho_j\) is the belief at
update \(j\), and \(H_{\max}\) is the maximum
run-length hypothesis count. The actor coefficient is
\(\beta_{\mathrm{eff}}=-2-c_{\mathrm{pen}}\lambda_w\).
The norm offset, sampled next actions, target network, actor, and symbolic
model are detached or held fixed for each critic target construction.
The abstract contraction result concerns a fixed actor and fixed penalty
fields; it does not establish convergence of these changing gradient updates.

\section{Symbolic fitting details}
\label{app:sindy_details}

The configured polynomial library is fit with STLSQ. Write
\(\mathsf X\) for joint sampled state-action inputs and \(\mathsf Y\) for regression targets.
In discrete mode the implemented target is the state increment
\(\Delta x_t=x_{t+1}-x_t\), not the next state itself:
\[
\mathsf Y\approx\boldsymbol\Phi(\mathsf X)\Xi^{\mathsf T},
\]
where the superscript \(\mathsf T\) denotes matrix transpose.
In continuous mode finite differences approximate the state derivative
using the sampling interval \(\Delta t\), changing the regression target
and its units. The frozen Torch wrapper evaluates
the polynomial terms and fitted coefficient matrix \(\Xi\), giving
\(f_{\mathrm{sym}}\) under the paper's discrete-mode convention.
The revised fit aligns every \(a_t\) with \(x_t\) and its increment before
concatenating episodes. Its wrapper uses the fitted polynomial library's
exponent matrix, including state-action cross terms and configured degrees
up to four. Both prediction paths require actions for control-aware models.
For degree-one, zero-intercept fits, the state coefficients describe
\(\widehat A-I\), and the action coefficients describe
\(\widehat B_{\mathrm{in}}\).
In LQR the environment applies force \(5a\), so the fitted input matrix
is in normalized-action units and the cost matrix is
\(R_{\mathrm{cost}}=25\operatorname{diag}(0.01,0.01)\);
\(Q_{\mathrm{cost}}=\operatorname{diag}(1,1,0.1,0.1)\).
The frozen reference uses the identified matrices with these task costs.
Prediction-quality diagnostics use the coefficient of determination
\(R^2\); they do not measure a uniform derivative error.
A prediction-quality threshold in the configuration is not a certificate
of \Cref{thm:sindy_jac}'s action-reference hypotheses.

\section{Arithmetic-head implementation}
\label{app:nau_details}

The historical feature extractor consists of two affine layers with LeakyReLU slope
\(0.01\), denoted \(G\). A fixed linear projection \(U_{\mathrm{quad}}\)
feeds the elementwise quadratic correction. The pre-squash mean is
\[
\mu_\theta(x)=
\eta_{\mathrm{mix}}W_{\mathrm{NAU}}G(x)
+(1-\eta_{\mathrm{mix}})
c_{\mathrm{quad}}\odot(U_{\mathrm{quad}}G(x))^{\odot2},
\]
where \(\odot\) denotes coordinatewise multiplication,
\(c_{\mathrm{quad}}\) is the vector of quadratic coefficients, and
\(\eta_{\mathrm{mix}}\in(0,1)\) is a learned sigmoid-transformed mixing
weight. Its unconstrained parameter is initialized at \(0.8\), giving
\(\eta_{\mathrm{mix}}=(1+\exp(-0.8))^{-1}\approx0.690\).
A separate head predicts the Gaussian log standard deviation.
For normalized actions the deterministic policy is \(\pi=\tanh\mu_\theta\).
An affine physical rescaling has offset \(b_a\) and scale \(c_a>0\);
in transit, \(h_{\mathrm{hold}}=b_a+c_a\pi\) with \(b_a=c_a=30\) s.

Training clamps \(W_{\mathrm{NAU}}\) to \([-1,1]\) and uses
\begin{equation}
\mathcal L_{\mathrm{NAU}}
=\frac{1}{N_W}\sum_{u,v}
\min\bigl(|(W_{\mathrm{NAU}})_{uv}|,\,
1-|(W_{\mathrm{NAU}})_{uv}|\bigr),
\label{eq:nau_regularizer}
\end{equation}
where \(u,v\) index matrix rows and columns and \(N_W\) is the number of
entries. This loss encourages discrete values but does not require exact
discreteness for the NAU derivative identity.
The regularizer in \eqref{eq:nau_regularizer} is the term used in
\eqref{eq:actor_loss}. The quadratic correction remains active in
\texttt{csbapr}, so this configuration is not a pure linear NAU actor.

\subsection{Complete composition calculation}
\label{app:leff}

Let \(G\) and \(H\) be continuously differentiable on neighborhoods of
the convex domains under consideration. Suppose
\(\|\mathrm DG\|\leq K_G\), \(\|\mathrm DH\|\leq K_H\), and their
derivatives have Lipschitz constants \(L_G,L_H\), respectively.
The derivative chain rule gives
\[
\begin{split}
\mathrm D(H\circ G)(x)-\mathrm D(H\circ G)(y)
={}&[\mathrm DH(G(x))-\mathrm DH(G(y))]\mathrm DG(x)\\
&+\mathrm DH(G(y))[\mathrm DG(x)-\mathrm DG(y)].
\end{split}
\]
Since \(\|G(x)-G(y)\|\leq K_G\|x-y\|\), its norm is at most
\((L_HK_G^2+K_HL_G)\|x-y\|\), proving
\eqref{eq:composition_rule}.
For the mixed head, a valid feature-space derivative Lipschitz bound is
\[
L_{H_{\mathrm{mix}}}
\leq2(1-\eta_{\mathrm{mix}})
\|c_{\mathrm{quad}}\|_\infty\|U_{\mathrm{quad}}\|_{\mathrm{op}}^2.
\]
Here the infinity norm is the maximum absolute vector entry. Even when the
head is linear, the \(K_HL_G\) term does not vanish unless the features
have constant derivatives.

For componentwise action squashing,
\(\sup|\tanh'|=1\) and
\(\sup|\tanh''|=4/(3\sqrt3)\).
One more application of the composition calculation yields
\begin{equation}
L_\pi\leq
\frac{4}{3\sqrt3}(K_HK_G)^2+L_HK_G^2+K_HL_G.
\label{eq:complete_composition}
\end{equation}
Equation~\eqref{eq:complete_composition} bounds the normalized deterministic
action map, including both feature curvature and tanh curvature.
An affine physical rescaling multiplies this derivative bound by \(c_a\).
The historical logged curvature surrogate omitted these global terms.
For example, \(\pi(x)=\tanh(\operatorname{LeakyReLU}_{0.01}(x))\)
has left and right derivatives \(0.01\) and \(1\) at zero. No finite
global derivative Lipschitz constant exists across that activation boundary.
Within a fixed activation region the feature map is affine, but such a
local calculation does not certify a segment that crosses regions.

\paragraph{Smooth arithmetic variant.}
The revised \texttt{smooth-nau} actor replaces both feature activations by
tanh, keeping the arithmetic and quadratic heads. Let \(W_1,W_2\) be the
two feature-layer matrices; their common hidden width is the feature
dimension \(d_z\). Define \(c_{\tanh}=4/(3\sqrt3)\). Since componentwise tanh has
derivative norm at most one and derivative Lipschitz constant
\(c_{\tanh}\), two applications of \eqref{eq:composition_rule} give
\[
\begin{aligned}
K_G&\leq\|W_2\|_{\mathrm{op}}\|W_1\|_{\mathrm{op}},\\
L_G&\leq c_{\tanh}\big[(\|W_2\|_{\mathrm{op}}\|W_1\|_{\mathrm{op}})^2
+\|W_2\|_{\mathrm{op}}\|W_1\|_{\mathrm{op}}^2\big].
\end{aligned}
\]
The features obey \(\|G(x)\|\leq R_G=\sqrt{d_z}\), where
\(R_G\) is their global radius bound. Thus
\[
K_H\leq\eta_{\mathrm{mix}}\|W_{\mathrm{NAU}}\|_{\mathrm{op}}+L_HR_G.
\]
Substituting these constants and the quadratic-head bound into
\eqref{eq:complete_composition} yields the implemented global certificate.
Feature biases do not affect the norm bounds because tanh remains bounded.
For \(\pi_{\mathrm{sym}}(x)=\tanh(-\widehat K_{\mathrm{fb}}x)\), the same
calculation gives reference curvature at most
\(c_{\tanh}\|\widehat K_{\mathrm{fb}}\|_{\mathrm{op}}^2\).
The boundary errors are computed as Euclidean action error and operator-norm
Jacobian error at \(x_0\), then inserted into \eqref{eq:ood_bound}.

\section{Cross-environment selection}
\label{app:irm_protocol}

For each fitted symbolic candidate \(\Xi\), let
\(\widehat E_e(\Xi)\) be its mean Euclidean prediction error in training
environment \(e\), evaluated on that environment's collected states and
targets. The implementation selects
\begin{equation}
\Xi_{\mathrm{sel}}\in
\operatorname*{arg\,min}_{\Xi\in\mathfrak C}
\operatorname{Var}_{e\in\mathcal E_{\mathrm{train}}}
[\widehat E_e(\Xi)],
\label{eq:invariance_selector}
\end{equation}
where \(\mathfrak C\) is the finite set of fitted candidates and
\(\Xi_{\mathrm{sel}}\) is the selected coefficient matrix.
The variance criterion \eqref{eq:invariance_selector} differs from
the earlier coefficient-averaging implementation as well as
IRMv1's gradient penalty. It can favor a model with uniformly large
prediction error. Neither derivative accuracy nor causal identification
follows from low variance alone. The default uses three environments;
their required number and diversity depend on the identification problem.

\newpage
\section{Defaults and experimental overrides}
\label{app:hyperparams}

Table~\ref{tab:hyperparams} lists shared implementation defaults.
Here \(\tau\) is the target-network interpolation coefficient,
\(\eta_{\mathrm{opt}}\) the optimizer step size,
\(N_{\mathrm{batch}}\) the replay-batch size, and
\(p_{\mathrm{haz}}\) the BOCD hazard probability.
Script overrides take precedence. The transit entrypoint has hidden width
64, state dimension 15, and action dimension 1; critic, batch, BC, and
symbolic-target overrides appear in the table. Stabilized actors use weight
decay \(10^{-4}\): NAU regularization has weight \(0.01\), KAN spline and
entropy regularization has weight \(0.001\), and MLPs have no head regularizer.
The pristine baseline uses zero weight decay and warmup, sets
\(c_{\mathrm{pen}}=5\), and disables the minimum target, normalization,
entropy floor, and rollout surprise. Transit experiments disable symbolic
filtering. \Cref{sec:limitations} distinguishes the historical legacy row.

\begin{table}[H]
\centering
\caption{Shared defaults and the primary transit overrides.}
\label{tab:hyperparams}
\small
\begin{tabular}{ll}
\toprule
Quantity & Value \\
\midrule
Discount \(\gamma\) & \(0.99\) \\
Target interpolation \(\tau\) & \(0.005\) \\
Initial entropy temperature \(\alpha\) & \(0.2\), automatically tuned \\
Optimizer step size \(\eta_{\mathrm{opt}}\) & \(3\times10^{-4}\) \\
Critics \(N_{\mathrm{crit}}\); batch size \(N_{\mathrm{batch}}\) & \(5;\ 256\); transit critics: \(3\) \\
Run-length count \(H_{\max}\); hazard \(p_{\mathrm{haz}}\) & \(20;\ 0.02\) \\
Norm offset \(\lambda_{\mathrm{reg}}\) & \(0.01\) \\
Critic spread \(\beta_{\mathrm{ood}}\) & \(0.1\) \\
BC \(\beta_{\mathrm{bc}}\) & \(0.001\); stabilized transit: \(0\) \\
Symbolic target \(\lambda_{\mathrm{sym}}\) & \(0.01\); transit: \(0\) \\
Actor consistency \(\lambda_{\mathrm{JC}}\) & \(0.01\); transit: \(0\) \\
NAU regularizer \(\lambda_{\mathrm{NAU}}\) & \(0.01\) \\
Entropy floor \(\alpha_{\min}\) & \(0.01\) \\
Penalty scale \(c_{\mathrm{pen}}\) & \(2\) \\
Recent rollout count \(N_{\mathrm{tail}}\) & \(256\) \\
Warmup; gradient-norm clipping & \(100\) updates; \(1.0\) \\
SINDy threshold; polynomial degree & \(0.1;\ 2\) \\
Exploration episodes; environments \(N_{\mathrm{env}}\) & \(50;\ 3\) \\
\bottomrule
\end{tabular}
\end{table}

\end{document}